\documentclass[letterpaper, 10 pt, journal, twoside]{IEEEtran}

\usepackage{amsmath,amsfonts}
\usepackage{array}
\usepackage[caption=false,font=normalsize,labelfont=sf,textfont=sf]{subfig}
\usepackage{textcomp}
\usepackage{stfloats}
\usepackage{url}
\usepackage{verbatim}
\usepackage{graphicx}
\def\BibTeX{{\rm B\kern-.05em{\sc i\kern-.025em b}\kern-.08em
    T\kern-.1667em\lower.7ex\hbox{E}\kern-.125emX}}
\usepackage{balance}

\usepackage{caption}
\usepackage{multirow}
\usepackage{hhline}

\usepackage{cite}

\usepackage{algorithm2e}
\RestyleAlgo{ruled}

\usepackage{amsthm}
\makeatletter
\def\@endtheorem{\endtrivlist}
\makeatother
\newtheorem{theorem}{Theorem}

\newtheorem{lemma}{Lemma}
\usepackage{amssymb}

\begin{document}

\markboth{IEEE Robotics and Automation Letters. Preprint Version. Accepted June, 2022}{Kim an Oh: Efficient Off-Policy Safe Reinforcement Learning
Using Trust Region Conditional Value at Risk} 

\title{Efficient Off-Policy Safe Reinforcement Learning \\
Using Trust Region Conditional Value at Risk}

\author{Dohyeong Kim and Songhwai Oh 
\thanks{Manuscript received: February, 24, 2022; Revised May, 27, 2022; Accepted May, 31, 2022.}
\thanks{This paper was recommended for publication by Editor Jens Kober upon evaluation of the Associate Editor and Reviewers' comments.
This work was supported in part by the Institute of Information \& Communications Technology Planning \& Evaluation under Grant 2019-0-01190, [SW Star Lab] Robot Learning: Efficient, Safe, and Socially-Acceptable Machine Learning, and in part by the National Research Foundation under Grant NRF-2022R1A2C2008239, both funded by the Korea government (MSIT).
\textit{(Corresponding authors: Songhwai Oh.)}}
\thanks{D. Kim and S. Oh are with the Department of Electrical and Computer Engineering and ASRI, Seoul National University, Seoul 08826, Korea (e-mail: dohyeong.kim@rllab.snu.ac.kr, songhwai@snu.ac.kr).}
\thanks{Digital Object Identifier (DOI): 10.1109/LRA.2022.3184793}
\thanks{\copyright 2022 IEEE.  Personal use of this material is permitted.  Permission from IEEE must be obtained for all other uses, in any current or future media, including reprinting/republishing this material for advertising or promotional purposes, creating new collective works, for resale or redistribution to servers or lists, or reuse of any copyrighted component of this work in other works.}
}

\maketitle
\begin{abstract}
This paper aims to solve a safe reinforcement learning (RL) problem with risk measure-based constraints.
As risk measures, such as conditional value at risk (CVaR), focus on the tail distribution of cost signals, constraining risk measures can effectively prevent a failure in the worst case.
An on-policy safe RL method, called TRC, deals with a CVaR-constrained RL problem using a trust region method and can generate policies with almost zero constraint violations with high returns. 
However, to achieve outstanding performance in complex environments and satisfy safety constraints quickly, RL methods are required to be sample efficient.
To this end, we propose an off-policy safe RL method with CVaR constraints, called off-policy TRC.
If off-policy data from replay buffers is directly used to train TRC, the estimation error caused by the distributional shift results in performance degradation.
To resolve this issue, we propose novel surrogate functions, in which the effect of the distributional shift can be reduced, and introduce an adaptive trust-region constraint to ensure a policy not to deviate far from replay buffers.
The proposed method has been evaluated in simulation and real-world environments and satisfied safety constraints within a few steps while achieving high returns even in complex robotic tasks.
\end{abstract}


\section{Introduction}
\IEEEPARstart{S}{afe} reinforcement learning (RL) addresses the problem of maximizing returns while satisfying safety constraints, so it has shown attractive results in a number of safety-critical robotic applications, such as locomotion for legged robots \cite{gangapurwala2020gcpo, bharadhwaj2021conservative} and safe robot navigation \cite{kim2022trc}.
Traditional RL methods prevent agents from acting undesirably through reward shaping, which is known as a time-consuming and laborious task.
In addition, the reward shaping can enlarge the optimality gap since reward shaping can be considered as fixing Lagrange multipliers in constrained optimization problems \cite{tessler2018reward}.
In contrast, safe RL methods \cite{kim2022trc, achiam2017constrained, xu2021crpo, yang2021wcsac} can achieve better safety performance than traditional RL methods because they directly solve optimization problems with explicitly defined safety constraints.

Safety constraints can be defined by various types of measures about safety signals, and several literatures \cite{chow2017risk, kim2022trc, yang2021wcsac} have shown that risk measures such as conditional value at risk (CVaR) can effectively reduce the likelihood of constraint violations.
Yang \textit{et al.} \cite{yang2021wcsac} have proposed a CVaR-constrained RL method using a soft actor-critic (SAC) \cite{haarnoja2018sac}, called worst-case SAC (WCSAC).
WCSAC utilizes aplenty of data from replay buffers to update a policy and deals with the constraint using the Lagrangian method.
However, WCSAC generates an overly conservative policy if constraints are excessively violated during the early training phase due to Lagrange multipliers (see the experiment section in \cite{kim2022trc}).
As several studies claim that the Lagrangian method makes training unstable \cite{xu2021crpo, liu2022constrained, Zhang2020firstcpo},
Kim and Oh \cite{kim2022trc} have proposed a trust region-based safe RL method for CVaR constraints called TRC, which shows the state-of-the-art performance in both simulations and sim-to-real experiments.
However, since TRC can only use on-policy trajectories for policy improvements, it has low sample efficiency.
Hence, there is a need for an efficient safe RL method which can take the advantage of the sample efficiency of off-policy algorithms.

In this paper, we propose a TRC-based off-policy safe RL method, called \textit{off-policy TRC}.
To update a policy using a trust-region method, including TRC, it is required to predict the constraint value of a random policy using collected data.
However, if using off-policy data for on-policy algorithms, prediction errors can be increased due to the distributional shift, and trained policies can breach constraints as demonstrated in Section \ref{sec:effectiveness of surrogate functions}.
To leverage off-policy trajectories for TRC without the distributional shift, we formulate novel surrogate functions motivated by a TRPO-based method in \cite{meng2021offtrpo}.
Then, we show how the upper bound of CVaR with off-policy data from a replay buffer can be estimated using the proposed surrogate functions under the assumption that the cumulative safety signals follow the Gaussian distribution.
Additionally, an adaptive trust-region constraint is derived, which indirectly lowers the estimation error by ensuring that the state distribution of the current policy does not deviate from the state distribution in the replay buffer.
By iteratively maximizing the lower bound of the objective while constraining the upper bound of CVaR in the trust region, it is possible to ensure a monotonic improvement of the objective while satisfying the constraint.
With various experiments in MuJoCo \cite{todorov2012mujoco}, Safety Gym \cite{ray2019benchmarking}, and real-robot environments, off-policy TRC shows excellent sample efficiency with high returns compared to previous methods.

Our main contributions are threefold.
First, we formulate the surrogate functions which leverage off-policy trajectories, and derive the upper bound of CVaR using the surrogate functions under the Gaussian assumption.
Second, we propose a practical algorithm for CVaR-constrained RL in off-policy manners with adaptive trust region, called off-policy TRC.
Finally, off-policy TRC is evaluated in a number of experiments both in simulations and real environments and shows outstanding performance as well as the lowest total number of constraint violations while significantly improving sample efficiency.

\section{RELATED WORK}

\subsection{Risk Measure-Constrained RL}

Safe RL methods generally set safety constraints using the expectation of the cumulative costs.
If expectation-based constraints are used, the Bellman operator for the traditional Markov decision process (MDP) can be applied to the safety critic, giving mathematical simplicity.
However, since these constraints focus only on the average case, it is challenging to prevent failures which occurred in the worst case.
For this reason, research on risk measure-based constraints is increasing to concentrate on the worst case.
Kim and Oh \cite{kim2022trc} and Yang \textit{et al.} \cite{yang2021wcsac} proposed methods constraining CVaR, which is a risk measure widely used in financial investment \cite{rockafellar2000optimization}, and estimated the CVaR using distributional safety critics.
Both methods resulted in fewer constraint violations than the expectation-constrained RL methods, and used the trust-region and Lagrangian methods to deal with constraints, respectively.
Methods proposed by Ying \textit{et al.} \cite{ying2021cppo} and Chow \textit{et al.} \cite{chow2017risk} also use the CVaR-based constraints but estimate CVaR using the sampling method proposed by Rockafellar \textit{et al.} \cite{rockafellar2000optimization}.
However, the estimation error can increase when the state space is high-dimensional because of the nature of sampling methods.
Instead of CVaR, methods proposed by Thananjeyan \textit{et al.} \cite{thananjeyan2021recovery} and Bharadhwaj \textit{et al.} \cite{bharadhwaj2021conservative} use chance constraints restricting the likelihood that the cumulative costs become above a specific value.
These methods calculate the violation likelihood and correct the action to lower the likelihood at every environmental interaction, which can cause a longer action execution time.

\subsection{Off-Policy Safe RL}

The usage of off-policy data not only improves the sample efficiency, but also reduces the optimality gap to reach higher final performance.\footnote{Since the MuJoCo tasks \cite{todorov2012mujoco} are dominated by off-policy algorithms as shown in the benchmarks of the Spinning-Up from OpenAI \cite{openai2018spinningup}, using off-policy data can reduce the optimality gap.}
Therefore, it is crucial to use off-policy data in safe RL to rapidly satisfy the constraints and maximize the return.
Wang \textit{et al.} \cite{wang2020csac} and Ha \textit{et al.} \cite{ha2021lsac} proposed safe RL methods based on the SAC method.
These methods use off-policy data to estimate the safety critic and utilize the Lagrangian approach to reflect the expectation-based constraints.
There is another method based on the Q-learning approach proposed by Huang \textit{et al.} \cite{huang2022multiobjective}, which deals with the safe RL problem with multiple objectives.
Liu \textit{et al.} \cite{liu2022constrained} transform the safe RL problem into a variational inference problem and proposes a new safe RL approach based on the expectation-maximization (EM) algorithm.
This method uses off-policy data in the E-step to fit a non-parametric distribution, and updates a policy in the M-step by minimizing the distance between the policy and the non-parametric distribution.

\subsection{Other Safe RL}
In addition to the methods mentioned above, there are various approaches to safe RL.
First, there are Lyapunov-based methods to train a safe policy \cite{chow2018lyapunov, sikchi2021lyapunov}. 
They learn the Lyapunov function using auxiliary costs and update policies to satisfy the Lyapunov property.
In \cite{zhao2021modelfree, Luo2021barrier}, control barrier functions are used to prevent agents from entering unsafe regions, and the control barrier functions can be constructed using trained transition models or prior knowledge of systems.
Also, there are projection-based methods that handle constraints through a primal approach \cite{xu2021crpo, Yang2020projection}.
They update policies to maximize the sum of reward and project the update direction to a safe policy set.

\section{BACKGROUND}

\subsection{Safe Reinforcement Learning}

We use constrained Markov decision processes (CMDPs) to define a safe reinforcement learning (RL) problem.
A CMDP is defined with a state space $\mathcal{S}$, an action space $\mathcal{A}$, a transition model $\mathcal{P}:\mathcal{S} \times \mathcal{A} \times \mathcal{S} \mapsto \mathbb{R}$, an initial state distribution $\rho$, a discount factor $\gamma$, a reward function $R:\mathcal{S}\times\mathcal{A}\times\mathcal{S}\mapsto\mathbb{R}$, a cost function $C:\mathcal{S}\times\mathcal{A}\times\mathcal{S}\mapsto\mathbb{R}_{\geq0}$, and a safety measure $\mathbf{S}$.
Given a policy $\pi$, value and advantage functions are defined as follows:
\begin{equation}
\label{eq:value and advantage}
\small
\begin{aligned}
V^{\pi}(s)&:=\underset{\pi,\mathcal{P}}{\mathbb{E}}\left[\sum_{t=0}^{\infty}\gamma^t R(s_t,a_t,s_{t+1})|s_0=s\right], \\
Q^{\pi}(s,a)&:=\underset{\pi,\mathcal{P}}{\mathbb{E}}\left[\sum_{t=0}^{\infty}\gamma^t R(s_t,a_t,s_{t+1})|s_0=s, a_0=a\right], \\
A^{\pi}(s,a) &:= Q^{\pi}(s,a) - V^{\pi}(s).
\end{aligned}
\end{equation}
The cost value $V_C^{\pi}$ and advantage $A_C^{\pi}$ can be defined by replacing the reward $R$ with the cost $C$ in (\ref{eq:value and advantage}).
Additionally, a discounted state distribution is defined as $d^{\pi}(s):=(1-\gamma)\sum_{t=0}^{\infty}\gamma^{t}P(s_t=s|\rho, \pi)$.
For safety constraints, we introduce a notion called the \textit{cost return}, which is used in \cite{kim2022trc, yang2021wcsac}:
\begin{equation}
\label{eq:cost return}
\small
\begin{aligned}
C_{\pi}=\sum_{t=0}^{\infty}\gamma^t C(s_t, a_t, s_{t+1}),
\end{aligned}
\end{equation}
where $s_0 \sim \rho$, $a_t \sim \pi(\cdot|s_t)$, and $s_{t+1}\sim \mathcal{P}(\cdot|s_t, a_t)$.
Then, we can define a safe RL problem as follows:
\begin{equation}
\label{eq:safe RL problem}
\small
\begin{aligned}
\underset{\pi}{\mathrm{maximize}}\underset{s\sim\rho}{\mathbb{E}}\left[V^{\pi}(s)\right] \;\; \mathbf{s.t.} \; \mathbf{S}(C_{\pi}) \leq d,
\end{aligned}
\end{equation}
where $d$ is a limit value, and the objective function is denoted by $J(\pi):=\underset{s\sim\rho}{\mathbb{E}}\left[V^{\pi}(s)\right]$.

\subsection{Conditional Value at Risk for Trust Region Method}

Kim and Oh \cite{kim2022trc} proposed a trust region-base safe RL method for conditional value at risk (CVaR) constraints, called \textit{TRC}.
CVaR is one of the commonly used risk measures in financial risk management and formulated as follows \cite{rockafellar2000optimization}:
\begin{equation}
\label{eq:true CVaR}
\small
\begin{aligned}
\mathrm{CVaR}_{\alpha}(X) = \underset{\nu}{\mathrm{min}}\left(\nu + \frac{1}{1 - \alpha}\mathbb{E}\left[(X - \nu)_{+}\right]\right),
\end{aligned}
\end{equation}
where $(\cdot)_+$ is a clipping function that truncates values below zero.
Because CVaR is calculated by conditional expectations above a certain level, it focuses on the worst case rather than the average case.
Therefore, constraining CVaR is a compelling way to prevent a worst-case failure.
By using CVaR for safety measures, TRC outperforms other expectation-constrained RL methods.
To formulate CVaR constraints, the followings are defined in \cite{kim2022trc}.
\begin{equation}
\small
\label{eq:square value}
\begin{aligned}
S_C^{\pi}(s)&:=\underset{\pi,\mathcal{P}}{\mathbb{E}}\left[C_{\pi}^2|s_0=s\right], \\ S_C^{\pi}(s, a)&:=\underset{\pi,\mathcal{P}}{\mathbb{E}}\left[C_{\pi}^2|s_0=s, a_0=a\right], \\
A_S^{\pi}&:=S_C^{\pi}(s, a) - S_C^{\pi}(s),
\end{aligned}
\end{equation}
where $S_C^{\pi}$ and $A_S^{\pi}$ are called cost square value and advantage, and expectations of the cost value and cost square value are denoted as $J_C(\pi):=\underset{s\sim\rho}{\mathbb{E}}\left[V_C^{\pi}(s)\right]$ and $J_S(\pi):=\underset{s\sim\rho}{\mathbb{E}}\left[S_C^{\pi}(s)\right]$, respectively.
Also, a doubly discounted state distribution is defined as $d_2^{\pi}:=(1-\gamma^2)\sum_{t=0}^{\infty}\gamma^{2t}P(s_t=s|\pi)$.
With state distributions, the expectations of the cost and cost square value can be expressed as follows \cite{kim2022trc}:
\begin{equation}
\small
\label{eq:J_C and J_S}
\begin{aligned}
J_{C}(\pi) &= \frac{1}{1-\gamma}\underset{d^{\pi}, \pi, \mathcal{P}}{\mathbb{E}}\left[C(s, a, s')\right], \\
J_S(\pi) &= \frac{1}{1-\gamma^2}\underset{d_2^{\pi}, \pi, \mathcal{P}}{\mathbb{E}}\left[C(s, a, s')^2 + 2\gamma C(s, a, s')V_C^{\pi}(s')\right].
\end{aligned}
\end{equation}
Assuming that the cost return follows the Gaussian distribution as in \cite{kim2022trc, yang2021wcsac}, CVaR can be expressed as follows:
\begin{equation}
\small
\label{eq:CVaR}
\mathrm{CVaR}_{\alpha}(C_{\pi}) = J_{C}(\pi) + \frac{\phi(\Phi^{-1}(\alpha))}{\alpha}\sqrt{J_{S}(\pi) - J_{C}(\pi)^2},
\end{equation}
where $\alpha$ is a risk level for CVaR, and $\phi$ and $\Phi$ are probability and cumulative density functions of the standard normal distribution, respectively.

\subsection{Off-Policy Trust Region Policy Optimization}

Meng \textit{et al.} \cite{meng2021offtrpo} improved trust region policy optimization (TRPO) \cite{schulman2015trust} using replay buffers and shown that monotonic improvement is guaranteed.
The method proposed in \cite{meng2021offtrpo}, called \textit{off-policy TRPO}, estimates an objective of a new policy $\pi'$ using a surrogate function which can utilize a large number of data from replay buffers.
The surrogate function is expressed as follows:
\begin{equation}
\small
\label{eq:off-trpo surrogates}
\begin{aligned}
J^{\mu, \pi}(\pi') &:= J(\pi) + \frac{1}{1-\gamma}\underset{d^{\mu},\mu}{\mathbb{E}}\left[\frac{\pi'(a|s)}{\mu(a|s)}A^{\pi}(s, a)\right],
\end{aligned}
\end{equation}
where $\mu$ is a behavioral policy and $\pi$ is a policy before being updated.
The objective function of $\pi'$ has a lower bound, which is derived using the surrogate function $J^{\mu, \pi}(\pi')$ as follows:
\begin{equation}
\label{eq:lower bound}
\small
\begin{aligned}
J(\pi') &\geq J^{\mu,\pi}(\pi') - \frac{4\epsilon_R\gamma}{(1 - \gamma)^2} D_{\mathrm{TV}}^{\mathrm{max}}(\mu||\pi')D_{\mathrm{TV}}^{\mathrm{max}}(\pi||\pi'),
\end{aligned}
\end{equation}
where $D_{\mathrm{TV}}^{\mathrm{max}}(\mu||\pi):=\underset{s}{\mathrm{max}}D_{\mathrm{TV}}(\mu(\cdot|s)||\pi(\cdot|s))$, $D_{\mathrm{TV}}$ is the total variation (TV) distance, and $\epsilon_R=\underset{s,a}{\mathrm{max}}\left|A^{\pi}(s,a)\right|$.
By iteratively maximizing the lower bound of $J(\pi')$, off-policy TRPO can guarantee that the objective is monotonically improved \cite{meng2021offtrpo}.
However, the surrogate function (\ref{eq:off-trpo surrogates}) is only valid for expectation-based objectives, so there is a need to develop new surrogate functions for objectives and constraints if a different safety measure is considered. 

\section{PROPOSED METHOD}

We aim to develop a CVaR-constrained RL method with high sample efficiency.
Therefore, we propose a trust region-based method in an off-policy manner which optimizes the following problem: 
\begin{equation}
\label{eq:CVaR constrained RL problem}
\small
\begin{aligned}
\underset{\pi}{\mathrm{maximize}}& \; J(\pi) \quad
\mathbf{s.t.} \; \mathrm{CVaR}_{\alpha}(C_{\pi}) \leq d/(1 - \gamma), \\
\end{aligned}
\end{equation}
where $d$ is a limit value.
Since (\ref{eq:CVaR constrained RL problem}) is maximizing the reward return while limiting the CVaR of the cost return for the policy $\pi$, it is necessary to estimate $J(\pi)$ and $\mathrm{CVaR}_{\alpha}(C_{\pi})$ using trajectories collected by a behavioral policy $\mu$.
Thus, we introduce surrogate functions for the constraint in (\ref{eq:CVaR constrained RL problem}) and approximate CVaR using the surrogate functions.
In the rest of this section, we show that the approximation is bounded and finally describe the proposed algorithm.

\subsection{Surrogate Functions}

Calculating CVaR of a random policy $\pi'$ using (\ref{eq:CVaR}) requires trajectories sampled by $\pi'$, which is computationally expensive.
Therefore, it is desirable to estimate CVaR using trajectories sampled from a behavioral policy $\mu$ and the current policy $\pi$ rather than calculate it directly.
To estimate the CVaR constraint, two surrogate functions are proposed as follows:
\begin{equation}
\small
\label{eq:surrogates}
\begin{aligned}
J_C^{\mu, \pi}(\pi') &:= J_C(\pi) + \frac{1}{1-\gamma}\underset{d^{\mu},\mu}{\mathbb{E}}\left[\frac{\pi'(a|s)}{\mu(a|s)}A_C^{\pi}(s, a)\right], \\
J_S^{\mu, \pi}(\pi') &:= J_S(\pi) + \frac{1}{1-\gamma^2}\underset{d_2^{\mu},\mu}{\mathbb{E}}\left[\frac{\pi'(a|s)}{\mu(a|s)}A_S^{\pi}(s, a)\right].
\end{aligned}
\end{equation}
In (\ref{eq:surrogates}), $J_C(\pi)$ and $J_S(\pi)$ can be calculated using (\ref{eq:J_C and J_S}) with trajectories sampled from $\pi$, and the other terms can be calculated using trajectories from $\mu$.
Using the proposed surrogate functions and (\ref{eq:CVaR}), CVaR of a random policy $\pi'$ can be approximated as follows:
\begin{equation}
\small
\label{eq:approximated CVaR}
\begin{aligned}
\overline{\mathrm{CVaR}}_{\alpha}(C_{\pi'}) \!:=\! J_{C}^{\mu,\pi}(\pi') \!+\! \frac{\phi(\Phi^{-1}(\alpha))}{\alpha}\sqrt{J_{S}^{\mu, \pi}(\pi') - J_{C}^{\mu, \pi}(\pi')^2}.
\end{aligned}
\end{equation}
The next section shows that the difference between the above approximation and the true value of CVaR is bounded.

\subsection{Upper Bound}

Before showing the upper bound of CVaR, the following notations are introduced for brevity:
\begin{equation}
\small
\label{eq:brevity}
\begin{aligned}
&D(\mu,\pi) := \underset{s}{\mathrm{max}}D_{\mathrm{TV}}(\mu(\cdot|s)||\pi(\cdot|s)), \\
\epsilon_C := &\underset{s,a}{\mathrm{max}}\left|A_C^{\pi}(s,a)\right|,\quad \epsilon_S := \underset{s,a}{\mathrm{max}}\left|A_S^{\pi}(s,a)\right|.
\end{aligned}
\end{equation}
Then, the upper bound of $\mathrm{CVaR}_{\alpha}(C_{\pi'})$ can be derived as follows.
\begin{theorem}
\label{theorem:upper bound}
For any polices $\mu$, $\pi$, and $\pi'$, define \\
$\epsilon_{\mathrm{CVaR}} := \frac{(\frac{4\epsilon_C\gamma}{1-\gamma^2})^2 D(\mu, \pi')D(\pi, \pi') + \frac{8\epsilon_C\gamma}{(1-\gamma)^2}J_C(\pi')+\frac{2\epsilon_S\gamma^2}{(1-\gamma^2)^2}}{\overline{\mathrm{CVaR}}_{\alpha}(C_{\pi'})}$.
Then, the following inequality holds:
\begin{equation}
\label{eq:upper bound}
\begin{aligned}
&\mathrm{CVaR}_{\alpha}(C_{\pi'}) \leq \overline{\mathrm{CVaR}}_{\alpha}(C_{\pi'}) \\
&+ \left(\frac{4\epsilon_C\gamma}{(1 - \gamma)^2} + \epsilon_{\mathrm{CVaR}}\frac{\phi(\Phi^{-1}(\alpha))}{\alpha}\right)D(\mu, \pi')D(\pi, \pi'),
\end{aligned}
\end{equation}
where the equality holds when $\pi = \pi'$.
\end{theorem}
The proof is given in Appendix \ref{appendix:upper bound}. 
Theorem \ref{theorem:upper bound} shows that $\overline{\mathrm{CVaR}}_{\alpha}(C_{\pi'})$ approximates $\mathrm{CVaR}_{\alpha}(C_{\pi'})$ well enough when $\pi$ and $\pi'$ are close to each other since $D(\mu, \pi')D(\pi, \pi')$ becomes sufficiently small.
In addition, as shown in the definition, $\epsilon_{\mathrm{CVaR}}$ is highly correlated to $\epsilon_C$ and $\epsilon_S$, and they become larger when the state change according to actions is huge.
Therefore, it is possible to decrease $\epsilon_{\mathrm{CVaR}}$ by designing an RL environment with a shorter time interval.

\subsection{Off-Policy TRC}

Directly calculating the objective and constraints of a random policy through sampling is computationally expensive, so we utilize their bounds to get the policy gradient efficiently.
Using the upper bound of CVaR derived in Theorem \ref{theorem:upper bound} and the lower bound of the objective in (\ref{eq:lower bound}), we can build a subproblem to obtain a new policy $\pi'$ given an old policy $\pi$ and a behavioral policy $\mu$ as follows:
\begin{equation}
\small
\label{eq:subproblem}
\begin{aligned}
&\;\; \underset{\pi'}{\mathrm{maximize}}\;J^{\mu, \pi}(\pi') - 4\epsilon_R\gamma D(\mu, \pi')D(\pi, \pi')/(1 - \gamma)^2\\ 
&\mathbf{s.t.} \; \overline{\mathrm{CVaR}}_{\alpha}(C_{\pi'}) + \\
&(\frac{4\epsilon_C\gamma}{(1 - \gamma)^2} \!+\! \epsilon_{\mathrm{CVaR}}\frac{\phi(\Phi^{-1}(\alpha))}{\alpha})D(\mu, \pi')D(\pi, \pi') \!\leq\! \frac{d}{1 - \gamma}.
\end{aligned}
\end{equation}
However, the TV distances $D(\mu, \pi')D(\pi, \pi')$ in the subproblem hinder training policies with large update steps, which is an issue raised in several trust region-related literature \cite{meng2021offtrpo, achiam2017constrained, schulman2015trust}.
Thus, we remove the TV distances from (\ref{eq:subproblem}) and add a trust region constraint as done in \cite{meng2021offtrpo, kim2022trc, schulman2015trust} to enlarge update steps.
For the trust region constraint, we convert the TV distances into KL divergences using the Pinsker's inequality and formulate the constraint as follows:
\begin{equation}
\small
\label{eq:trust region}
D_{\mathrm{KL}}(\pi||\pi') + \delta_{\mathrm{old}} \leq \delta,
\end{equation}
where {\small $\delta_{\mathrm{old}} = \sqrt{D_{\mathrm{KL}}(\mu||\pi)(\delta + D_{\mathrm{KL}}(\mu||\pi)/4)} - D_{\mathrm{KL}}(\mu||\pi)/2$}, {\small $D_{\mathrm{KL}}(\pi||\pi')=\underset{s\sim d_{\mu}}{\mathbb{E}}\left[D_{\mathrm{KL}}(\pi(\cdot|s)||\pi'(\cdot|s))\right]$}, and $\delta$ is a constant for the trust region size.
The derivation is presented in Appendix \ref{appendix:trust region}.
When the KL divergence between $\pi$ and $\mu$ is large enough, $\delta - \delta_{\mathrm{old}}$ becomes nearly zero, so the trust region constraint ensures that the policy is updated not far from the replay buffer.
Hence, the subproblem (\ref{eq:subproblem}) can be reformulated as follows:
\begin{equation}
\small
\label{eq:final subproblem}
\begin{aligned}
&\qquad \underset{\pi'}{\mathrm{maximize}}\;J^{\mu, \pi}(\pi') \\ 
& \mathbf{s.t.} \; \overline{\mathrm{CVaR}}_{\alpha}(C_{\pi'}) \leq d/(1 - \gamma), \\
&\quad D_{\mathrm{KL}}(\pi||\pi') + \delta_{\mathrm{old}} \leq \delta.
\end{aligned}
\end{equation}
The subproblem (\ref{eq:final subproblem}) is nonconvex, so we approximate the objective and the CVaR constraint as linear and the trust region as quadratic, and obtain $\pi'$ using linear and quadratic constrained linear programming (LQCLP) as in \cite{kim2022trc, achiam2017constrained}.
If the feasibility set of (\ref{eq:final subproblem}) is empty, we solely minimize the approximated CVaR under the trust region constraint.

To update the value and cost value networks, we use the following retrace estimators as target values \cite{meng2021offtrpo, wang2017acer}.
\begin{equation}
\small
\label{eq:retrace for value}
\begin{aligned}
\overline{V}_t &= r_t + \gamma V^{\pi}(s_{t+1}) + \gamma\lambda\overline{\rho}_{t+1}(\overline{V}_{t+1} - V^{\pi}(s_{t+1})), \\
\overline{V}_{C, t} &= c_t + \gamma V_{C}^{\pi}(s_{t+1}) + \gamma\lambda\overline{\rho}_{t+1}(\overline{V}_{C, t+1} - V_{C}^{\pi}(s_{t+1})), \\
\end{aligned}
\end{equation}
where $r_t=R(s_t, a_t, s_{t+1})$, $c_t=C(s_t, a_t, s_{t+1})$, $\overline{\rho}_{t} = \mathrm{min}(1, \frac{\pi(a_t|s_t)}{\mu(a_t|s_t)})$, and $\lambda$ is a trace-decay value.
The retrace estimator for the cost square value can also be defined as follows.
\begin{equation}
\small
\label{eq:retrace for square value}
\begin{aligned}
\overline{V}_{S, t} = & \; c_t^2 + 2\gamma c_t V_{C}^{\pi}(s_{t+1}) + \gamma^2 V_{S}^{\pi}(s_{t+1}) \\
&+ \gamma^2\lambda\overline{\rho}_{t+1}(\overline{V}_{S, t+1} - V_{S}^{\pi}(s_{t+1})). \\
\end{aligned}
\end{equation}
Then, the value, cost value, and cost square networks are updated by minimizing the mean squared error between the target values. 
The overall process of off-policy TRC is summarized in Algorithm \ref{algo:proposed algorithm}.

\begin{algorithm}[!t]
\small
\caption{Off-Policy TRC}
\label{algo:proposed algorithm}
\KwData{policy network $\pi(\cdot|s;\theta)$, value network $V^{\pi}(s;\phi)$, cost value network $V_C^{\pi}(s;\phi_C)$, cost square network $S_C^{\pi}(s;\phi_S)$, batch size $B$, length of replay buffer $L$, and collect steps $S$.}
Initialize network parameters $\theta$, $\phi$, $\phi_C$, and $\phi_S$. \\
Initialize replay buffer $\mathcal{D}$ of length $L$. \\
\For{epochs=1, P}{
    Initialize rollout buffer $\mathcal{R}$. \\
    \For{t=1, S}{
        Sample an action $a_t$ and calculate the sampling probability $\mathrm{prob}_t$ from $\pi(\cdot|s_t;\theta)$. \\
        Take the action $a_t$ in the environment and get reward $r_t$, cost $c_t$, and next state $s_{t+1}$. \\
        Store $(s_t, a_t, \mathrm{prob}_t, r_t, c_t, s_{t+1})$ in $\mathcal{R}$. \\
    }
    Using (\ref{eq:J_C and J_S}), calculate $J_C(\pi)$ and $J_S(\pi)$ with $\mathcal{R}$. \\
    Concatenate $\mathcal{D}$ with $\mathcal{R}$. \\
    Sample trajectories $\mathcal{T}$ with length of $B$ from $\mathcal{D}$. \\
    Update policy paramterers $\theta$ by solving (\ref{eq:final subproblem}) with $\mathcal{T}$. \\
    Calculate target values in (\ref{eq:retrace for value}), (\ref{eq:retrace for square value}) and update value $\phi$, cost value $\phi_C$, and cost square $\phi_S$ parameters with the targets.
}
\end{algorithm}

\section{EXPERIMENTS}

We aim to answer the following questions through experiments:
\textbf{1)} Does off-policy TRC provide better performance and higher sample efficiency than other safe RL methods? \textbf{2)} Can it be applied to robots with different dynamics? \textbf{3)} Are the newly defined surrogate functions valid for off-policy data?
To answer these questions, we set up simulations and sim-to-real experiments with various types of robots and use several safe RL baseline methods. 
Next, the effectiveness of the surrogate function is evaluated by comparing off-policy TRC with a variant of TRC, which considers off-policy data as on-policy data.
In addition, experiments are also performed on different settings of the replay buffer-related hyperparameters to examine the effect of off-policy data on training.

\subsection{Simulation Setup}

\subsubsection{MuJoCo}

For robotic locomotion tasks, we use \texttt{HalfCheetah-v2} and \texttt{Walker2d-v2} provided by the MuJoCo simulator \cite{todorov2012mujoco}.
For stable movement in \texttt{HalfCheetah-v2}, the following cost function, which penalizes the agent if the angle of the torso is above a specific value, is defined.
\begin{equation}
\small
\label{eq:cost function of mujoco}
\begin{aligned}
C_{\mathrm{cheetah}}(s, a, s') := \mathrm{sigmoid}(k(|\theta_{\mathcal{T}}| - b)),
\end{aligned}
\end{equation}
where $\theta_{\mathcal{T}}$ is the angle of the torso, and $k$ ($=10$) and $b$ ($=\pi/4$) are constants.
In \texttt{Walker2d-v2}, to make the height of the center of mass away from the ground, the following cost function is defined.
\begin{equation}
\small
\label{eq:cost function of mujoco2}
\begin{aligned}
C_{\mathrm{walker}}(s, a, s') := \mathrm{sigmoid}(k(b - h_{\mathrm{CoM}})),
\end{aligned}
\end{equation}
where $h_{\mathrm{CoM}}$ is the height of the center of mass, and $k=15$, and $b=0.5$.
In the both tasks, agents are trained without early termination.

\subsubsection{Safety Gym}

The Safety Gym \cite{ray2019benchmarking} provides multiple robots and tasks for safe RL, and we use the following tasks: \texttt{Safexp-PointGoal1-v0}, \texttt{Safexp-CarGoal1-v0}, and \texttt{Safexp-DoggoGoal1-v0} with two modifications. 
The goal information is provided in a LIDAR form in the original state, which is not a realistic setting, so we modify the state to include the goal position.
Additionally, instead of providing a binary cost signal indicating whether the robot is in a hazard area or not, we use a soft cost signal which is defined by replacing $h_{\mathrm{CoM}}$ in (\ref{eq:cost function of mujoco2}) with the minimum distance to obstacles ($k=10$, $b=0.2$).
The other settings including the reward function are the same as the original.

\begin{figure}[t]
\centering
\hfill
\subfloat[Jackal Robot]%
{
    \label{sfig:real jackal}
    {
        \includegraphics[width=0.3\linewidth]{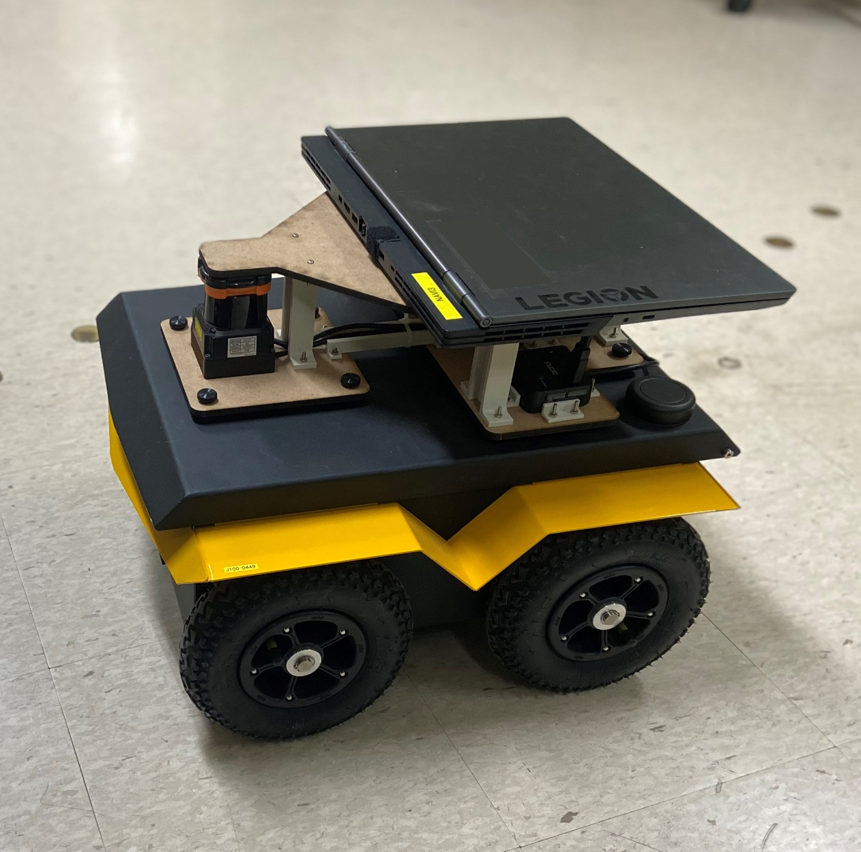}
    }
}\hfill
\subfloat[Jackal Evaluation Task]%
{
    \label{sfig:task diagram}
    {
        \includegraphics[width=0.42\linewidth]{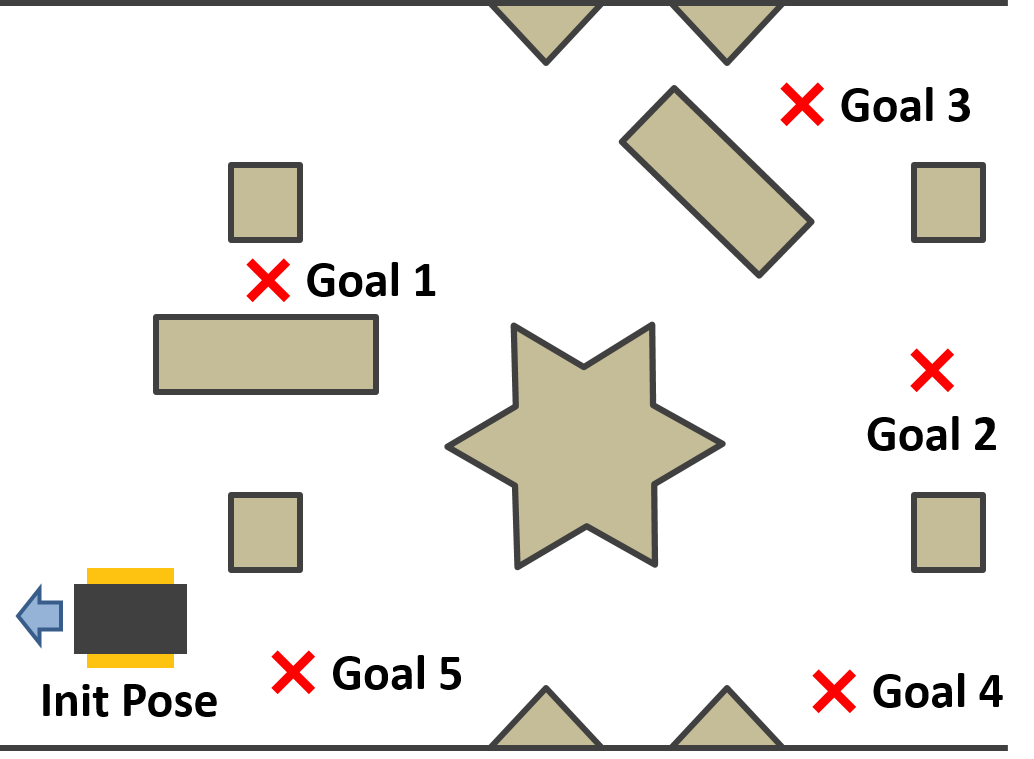}
    }
}\hfill
\caption{Jackal robot and the real-world evaluation task.
The black lines represent walls and the brown boxes represent randomly placed obstacles. 
The target goal is assigned to the next position each time the robot arrives.
}
\label{fig:real_task}
\vspace{-10pt}
\end{figure}

\begin{figure*}[t]
\centering
\subfloat[Half-Cheetah]
{
    \label{sfig:halfcheetah result}
    {
        \includegraphics[width=0.172\linewidth]{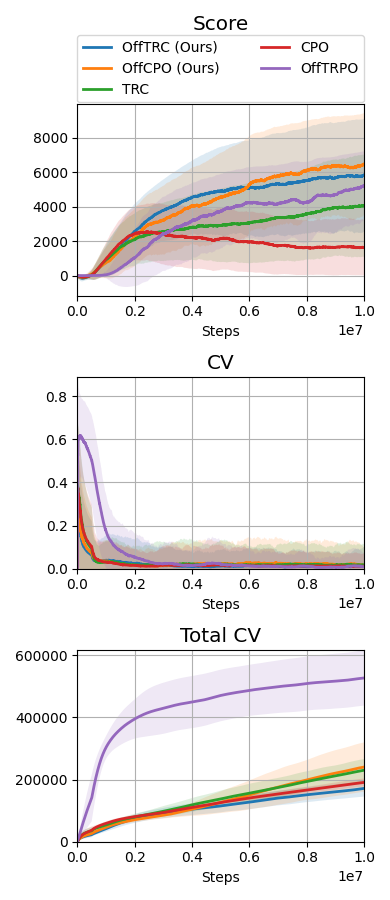}
    }
}\hfill
\subfloat[Walker-2d]
{
    \label{sfig:walker2d result}
    {
        \includegraphics[width=0.172\linewidth]{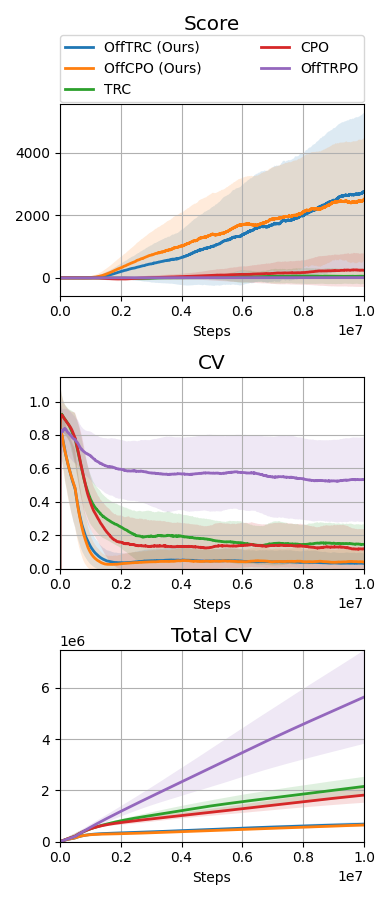}
    }
}\hfill
\subfloat[Point Goal]
{
    \label{sfig:point result}
    {
        \includegraphics[width=0.172\linewidth]{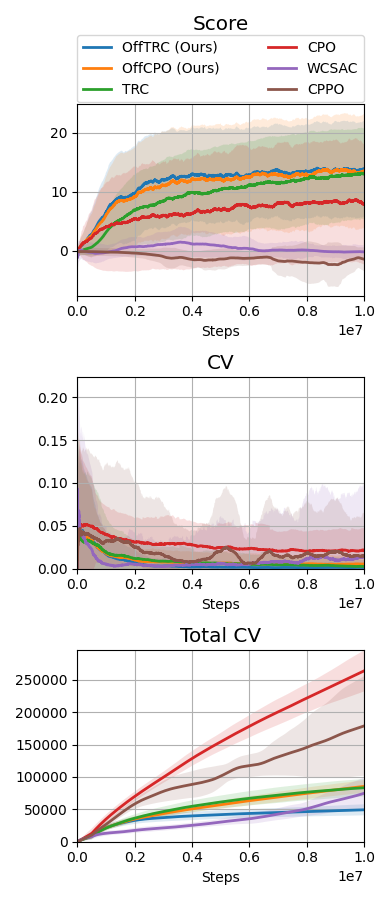}
    }
}\hfill
\subfloat[Car Goal]
{
    \label{sfig:car result}
    {
        \includegraphics[width=0.172\linewidth]{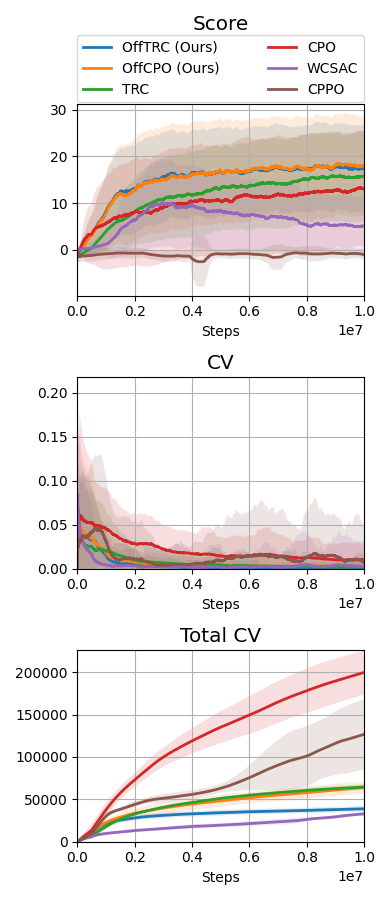}
    }
}\hfill
\subfloat[Doggo Goal]
{
    \label{sfig:doggo result}
    {
        \includegraphics[width=0.172\linewidth]{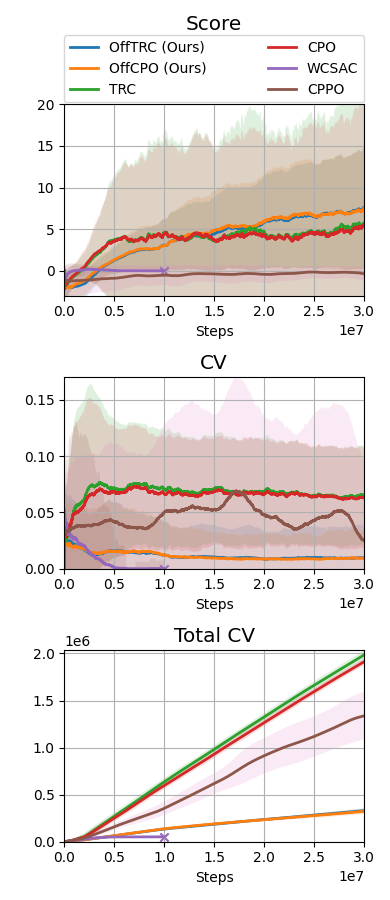}
    }
}
\caption{Training curves of the simulation experiments. 
The the top row shows the scores for different tasks, the middle row shows the number of constraint violations (CVs) divided by the episode length, and the bottom row shows the total number of CVs during training.
Each graph is obtained by training with five random seeds.
}
\label{fig:simulation results}
\vspace{-10pt}
\end{figure*}

\begin{figure}[!t]
\centering
\subfloat[Jackal simulation training curves.]
{
    \label{sfig:jackal sim result}
    {
        \includegraphics[width=0.4\textwidth]{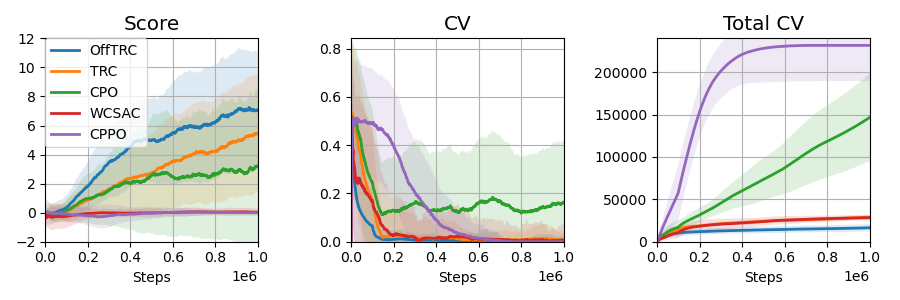}
    }
}\hfill
\subfloat[Real-world evaluation results. 
Bars represent mean values, and error bars represent standard deviations. 
Note that both OffTRC and TRC show zero CV and zero failures.]
{
    \label{sfig:evaluation result}
    {
        \includegraphics[width=0.4\textwidth]{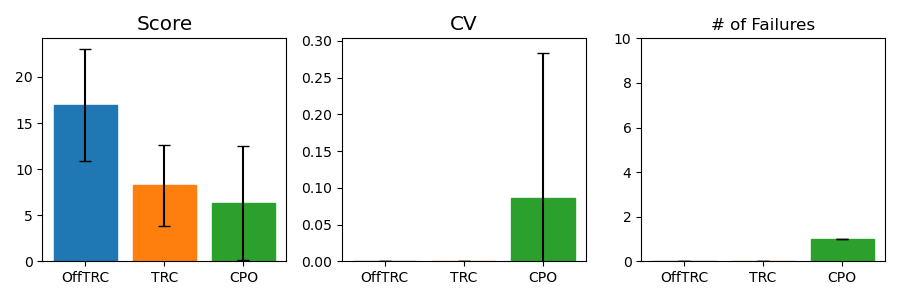}
    }
}\hfill
\caption{Training curves of the Jackal simulation and the real-world evaluation result. 
The number of failures is counted when the distance between the robot and obstacles is lower than $0.2 m$.
The Evaluation results are obtained by averaging the results of 10 episodes.}
\label{fig:jackal results}
\vspace{-10pt}
\end{figure}

\subsection{Sim-to-Real Experiment Setup}

To verify that off-policy TRC is applicable to real robots, we conduct a sim-to-real experiment on a UGV robot, Jackal from Clearpath \cite{clearpath2015jackal}, shown in Fig. \ref{sfig:real jackal}.
Agents are first trained in a MuJoCo simulator, where the state space is a 32-dimensional space consisting of LIDAR values, linear and angular velocities of the robot, and the goal position, and the action space is a two-dimensional space consisting of the linear acceleration and angular velocity.
Eight obstacles are spawned at random positions in the beginning of each episode, and the reward and cost function are the same as the Safety Gym.
Once trained in simulation, agents are evaluated in the real world without additional training on a task shown in Fig. \ref{sfig:task diagram}, where goals are spawned sequentially.
If the agent gets closer than $0.2\;m$ to obstacles, a failure occurs and the episode ends.

In all tasks, the number of constraint violations (CVs) is counted when $C(s_t, a_t, s_{t+1}) \geq 0.5$, and the following score metric is used as in \cite{kim2022trc}.\footnote{The score metric is defined by reward sum divided by the number of CVs. Since the reward sum increases as safety is ignored, this metric is used for fair comparisons.}
\begin{equation}
\small
\mathrm{Score}:=\sum_{t=0}^{T-1}R_t \bigg/ \left(1 + \sum_{t=0}^{T-1}\mathrm{CV}_t\right).
\end{equation}
The policy and all value functions are modeled by neural networks with two hidden layers of 512 nodes, and the activation functions are \texttt{ReLU}.
The learning rate for all value networks is $0.0002$ and $\delta$ for the trust region is $0.001$.
We update networks at every collect steps $S$ ($=1000$) using trajectories of length $B$ ($=5000$) sampled from the replay buffer of length $L$ ($=50000$).
Exceptionally, we set $B=20000$ and $L=100000$ for the MuJoCo and doggo goal tasks.
For constraints, we use $\alpha$ for CVaR as $0.125$ or $1.0$ and $d$ as $0.025$.\footnote{
To give a tip for how to set $\alpha$ and $d$, if you want the agent to violate safety less than $25$ times out of $1000$ interaction steps with $95\%$ probability, set $d$ as $25.0/1000.0$ and $\alpha$ as $f^{-1}(\Phi^{-1}(0.95))$, where $f(x)=\phi(\Phi^{-1}(x))/x$.}

\subsection{Comparison with Baselines}

Throughout the experiments, we evaluate the proposed method, off-policy TRC, with $\alpha=0.125$, indicated as \textit{OffTRC}, and with $\alpha=1.0$, indicated as \textit{OffCPO}.
If $\alpha=1.0$, off-policy TRC becomes the off-policy version of constrained policy optimization (CPO) \cite{achiam2017constrained}, as the CVaR becomes an expectation if $\alpha=1.0$.
In the MuJoCo experiment, to show the sample efficiency, the off-policy TRC is compared to other trust region-based methods: CPO \cite{achiam2017constrained} and TRC \cite{kim2022trc}.
In addition, to check how the constraints defined in (\ref{eq:cost function of mujoco}) and (\ref{eq:cost function of mujoco2}) have an effect on the locomotion tasks, the traditional RL method, off-policy trust region policy optimization (OffTRPO) \cite{meng2021offtrpo} is also used.
For the Safety Gym and the sim-to-real experiments, we use the following risk measure-constrained RL methods as baselines: CVaR-proximal policy optimization (CPPO) \cite{ying2021cppo}, WCSAC \cite{yang2021wcsac}, and TRC \cite{kim2022trc}.
CPPO estimates CVaR from sampling, while WCSAC and TRC estimate CVaR using distributional safety critics.
To compare with expectation-based safe RL methods, CPO is also used as a baseline.

The MuJoCo and Safety Gym simulation results are shown in Fig. \ref{fig:simulation results}, and the Jackal simulation and real-world evaluation results are shown in Fig. \ref{fig:jackal results}.
WCSAC and CPPO are excluded from the real-world evaluation because the agents trained by them do not move but freeze at initial positions in the Jackal simulation.
In all tasks, including the real-world experiment, the proposed methods, which are OffTRC and OffCPO, show the highest scores and the lowest total number of CVs.
OffTRC and OffCPO show the similar level in scores, but OffTRC shows lower total numbers of CVs than OffCPO in half-cheetah, point goal, and car goal tasks.
In the walker and doggo tasks, it is difficult to satisfy the constraints due to the unstable dynamics, so the number of CVs of OffTRC and OffCPO are similar.
However, OffCPO scores higher than OffTRC in the half-cheetah task, which means that the CVaR constraints make it difficult to learn tasks with stable dynamics.
Observing that the scores increase the fastest and the total number of CVs are significantly low, the proposed methods show excellent sample efficiency by simultaneously increasing the return while satisfying the safety constraints.
OffTRPO results in the third highest score in the half-cheetah task, showing that the number of CVs converges to zero.
As the dynamics model of the half-cheetah is stable, OffTRPO can train policies well without any constraints.
Nevertheless, OffTRPO shows the worst performance in the walker-2d task, inferring that restricting the height of CoM helps training.
CPO shows the third-highest score in the walker-2d, but there is a significant score gap between off-policy TRC and CPO, which indicates that sample efficiency is essential to achieve outstanding performance.
Additionally, CPO shows excessive numbers of CVs in all other tasks and records a failure in the real-world evaluation, as shown in Fig. \ref{sfig:evaluation result}.
It means that risk measure-constrained RL methods rather than expectation-constrained RL methods are required to prevent failures.
TRC shows the third-best performance for all tasks except the MuJoCo tasks.
Because TRC is an on-policy RL algorithm, it is difficult to train complex robots such as multi-joint robots, since these robotic tasks require large quantities of training data.
Still, as TRC is a trust-region-based safe RL method, it shows monotonic performance improvement and low CVs.
WCSAC uses the Lagrangian method to handle the CVaR constraint and is one of the off-policy RL methods, so it can be expected to show high sample efficiency.
However, WCSAC shows low scores in all tasks and synthesizes immobile policies in the doggo and Jackal tasks because the Lagrangian method makes training unstable. 
CPPO is also one of the CVaR-constrained RL method and estimate CVaR using a sampling method.
However, CPPO suffers performance degradation because it uses a sampling-based estimation for CVaR with large variances and the Lagrangian approach to handle the constraints.

\subsection{Effectiveness of Surrogate Functions}
\label{sec:effectiveness of surrogate functions}

To show that the proposed surrogate functions (\ref{eq:surrogates}) are effective for off-policy data, we perform an experiment comparing a variant of TRC against off-policy TRC.
The TRC variant treats off-policy data as on-policy data and updates the policy via the TRC method.
The result is presented in Fig. \ref{fig:effectiveness}.
In the training curve, the CVs of the variant are high relative to the off-policy TRC, which can be considered estimation errors due to the distributional shift in the training data.
Hence, we can conclude that the CVaR estimation through the proposed surrogate functions (\ref{eq:surrogates}) is critical for utilizing off-policy data.

\begin{figure}[t]
\centering
\includegraphics[width=0.4\textwidth]{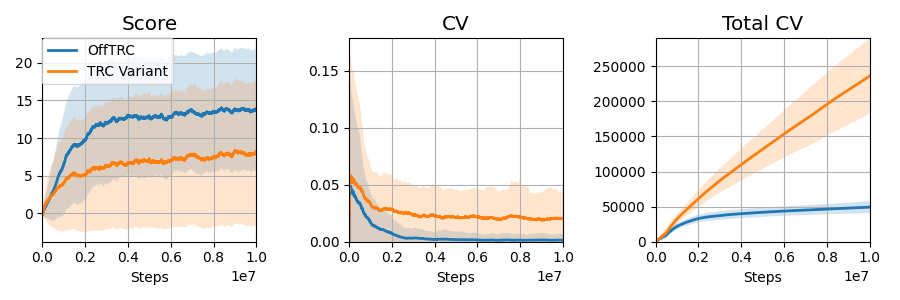}
\caption{Comparison with off-policy TRC and a variant of TRC to show Effectiveness of the surrogate functions. The TRC variant is trained using TRC with off-policy data.}
\label{fig:effectiveness}
\vspace{-10pt}
\end{figure}

\subsection{Ablation Study}

Since the off-policy TRC uses off-policy data from a replay buffer, it is required to analyze how the parameters related to the replay buffer affect learning.
The replay buffer-related parameters are the batch size (the length of sampled trajectories for policy update), the length of the replay buffer, and the collect steps (the policy is updated for each collect step).
We train the off-policy TRC on the point goal task for each parameter with three different values, and the results are shown in Table \ref{table:ablation}.
The parameter with the greatest influence on training is the batch size.
The larger the batch size, the lower the total number of CVs.
As the batch size increases, the search direction for policy update becomes accurate, so the agent can effectively lower the constraint value at the early training phase.
Next, the collect steps affect the length of on-policy data and the number of policy updates.
Since on-policy data is used to calculate $J_C(\pi)$ and $J_S(\pi)$, constituting the approximated CVaR $\overline{\mathrm{CVaR}}_{\alpha}(C_{\pi'})$, the approximation error increases as the amount of on-policy data decreases.
Thus, the larger the collect steps, the lower the total number of CVs, but also the lower the reward sum, as shown in Table \ref{table:ablation}.
Lastly, while the length of the replay buffer has a little effect on performance, if it is too short, the total number of CVs can increase slightly. 

To give an intuition for the risk level $\alpha$, experiments with different risk levels for TRC and off-policy TRC are also conducted, and the result is shown in Fig. \ref{fig:pareto}.
Observing that policies with high-risk levels are located in the upper right corner of the figure, a high return but risky policy can be obtained with a high-risk level.

\begin{table}[t]
\centering
\resizebox{0.9\columnwidth}{!}{%
\setlength\doublerulesep{1.0pt}
\begin{tabular}{ll|l|l|l}
\hline
\multicolumn{2}{l|}{}                                     & Reward Sum $\uparrow$    & CV (mean, std) $\downarrow$         & Total CV $\downarrow$      \\ \hhline{=====}
\multicolumn{1}{l|}{\multirow{3}{*}{Batch size}}    & 2e3 & \textbf{17.073} & 0.0024, 0.0067          & 85211          \\ \cline{2-5} 
\multicolumn{1}{l|}{}                               & 5e3 & 14.616          & \textbf{0.0007, 0.0035} & 45238          \\ \cline{2-5} 
\multicolumn{1}{l|}{}                               & 1e4 & 14.550          & 0.0007, 0.0037          & \textbf{27951} \\ \hhline{=====}
\multicolumn{1}{l|}{\multirow{3}{*}{Collect steps}} & 5e2 & \textbf{15.270} & 0.0016, 0.0058          & 42945          \\ \cline{2-5} 
\multicolumn{1}{l|}{}                               & 1e3 & 14.616          & 0.0007, 0.0035          & 45238          \\ \cline{2-5} 
\multicolumn{1}{l|}{}                               & 2e3 & 13.872          & \textbf{0.0006, 0.0041} & \textbf{39407} \\ \hhline{=====}
\multicolumn{1}{l|}{\multirow{3}{*}{Replay length}} & 2e4 & 14.934          & 0.0017, 0.0061          & 52788          \\ \cline{2-5} 
\multicolumn{1}{l|}{}                               & 5e4 & 14.616          & \textbf{0.0007, 0.0035} & \textbf{45238} \\ \cline{2-5} 
\multicolumn{1}{l|}{}                               & 1e5 & \textbf{15.072} & 0.0008, 0.0047          & 45731          \\ \hline
\end{tabular}
}
\caption{Ablation study for the replay buffer-related parameters.
Each cell value is average from five different seed For each parameter, and CV is divided by the length of the episode.}
\label{table:ablation}
\vspace{-10pt}
\end{table}

\begin{figure}[t]
\centering
\includegraphics[width=0.3\textwidth]{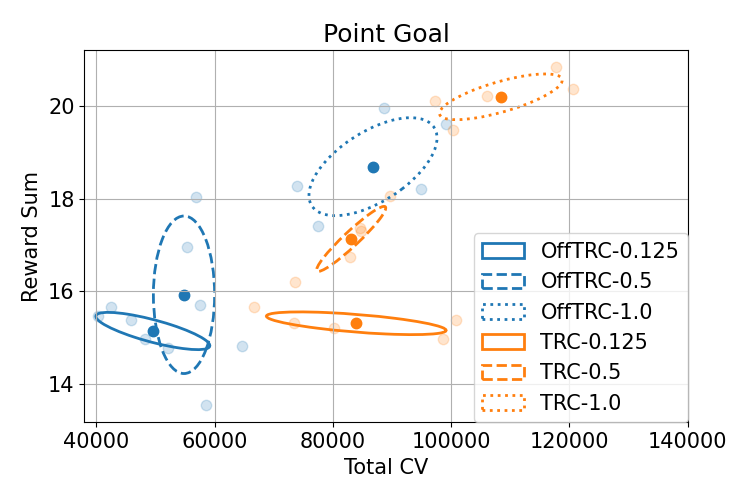}
\caption{Reward sum and total number of CVs on the point goal task with different risk levels $\alpha$ of CVaR.
The risk level for each run is indicated in the suffix in the legend.
}
\label{fig:pareto}
\vspace{-10pt}
\end{figure}

\section{CONCLUSIONS}

In this paper, a sample efficient off-policy safe RL algorithm using trust region CVaR, called off-policy TRC, is presented. 
We have proposed novel surrogate functions such that the CVaR constraint can be estimated using off-policy data without the distributional shift, as shown in Section \ref{sec:effectiveness of surrogate functions}.
In simulation and real-world experiments, the proposed off-policy TRC has achieved the highest returns with the lowest number of constraint violations in all tasks, showing its high sample efficiency.

\appendix

\small

In this section, we assume that the state space $\mathcal{S}$ and action space $\mathcal{A}$ are finite spaces.
Thus, $d^{\pi}$, $d_2^{\pi}$, and $\rho$ are treated as vectors in $\mathbb{R}^{|\mathcal{S}|}$.

\subsection{Preliminary}

\begin{lemma}
\label{lemma:square upper bound}
Given stochastic policies $\pi$, $\pi'$ and a behavior policy $\mu$, the following inequality holds: 
\begin{equation}
\small
\begin{aligned}
\left|J_{S}(\pi') - J_{S}^{\mu, \pi}(\pi')\right| \leq\frac{2\epsilon_S \gamma^2}{(1 - \gamma^2)^2}D(\mu, \pi')D(\pi, \pi'),
\end{aligned}
\end{equation}
where $\epsilon_{S} := \underset{s,a}{\mathrm{max}}\left|A_{S}^{\pi}(s,a)\right|$.
\end{lemma}

\begin{proof}
For any function $f$, let define the following variable:
\begin{equation*}
\small
\begin{aligned}
&\delta_f^{\pi}(s)\!=\!\!\underset{\pi, \mathcal{P}}{\mathbb{E}}\left[C(s, a, s')^2\! + \gamma^2f(s') - f(s) + 2\gamma C(s,a,s')V_C^{\pi}(s')\right]. \\
\end{aligned}
\end{equation*}
Then, the following equation holds by Corollary 1 in \cite{kim2022trc}:
\begin{equation*}
\small
\begin{aligned}
J_{S}(\pi) &= \langle \rho, f \rangle + \frac{1}{1 - \gamma^2}\langle d_2^{\pi}, \delta_{f}^{\pi}\rangle.
\end{aligned}
\end{equation*}
By substituting $f$ with $S_C^{\pi}$,
\begin{equation*}
\small
\begin{aligned}
J_{S}^{\mu, \pi}(\pi') &= J_S(\pi) + \frac{1}{1 - \gamma^2}\langle d_2^{\mu}, \delta_{f}^{\pi'}\rangle \\
&= \langle \rho, f\rangle + \frac{1}{1 - \gamma^2}\langle d_2^{\pi}, \delta_{f}^{\pi}\rangle + \frac{1}{1 - \gamma^2}\langle d_2^{\mu}, \delta_{f}^{\pi'}\rangle. \\
\end{aligned}
\end{equation*}
\begin{equation*}
\small
\begin{aligned}
\Rightarrow\left|(1 - \gamma^2)(J_{S}(\pi') - J_{S}^{\mu, \pi}(\pi'))\right| &= \left|\langle d_{2}^{\pi'} - d_{2}^{\mu}, \delta_{f}^{\pi'} \rangle\right| \\
&= \left|\langle d_{2}^{\pi'} - d_{2}^{\mu}, \delta_{f}^{\pi'} - \delta_{f}^{\pi} \rangle\right| \\
&\leq \frac{2\epsilon_{S}\gamma^2}{1 - \gamma^2}D(\mu, \pi')D(\pi, \pi'). \\
\end{aligned}
\end{equation*}
\end{proof}

\subsection{Proof of Theorem \ref{theorem:upper bound}}
\label{appendix:upper bound}
\begin{proof}
By Theorem 1 in \cite{meng2021offtrpo}, the following inequality holds:
\begin{equation}
\small
\label{eq: cost inequality}
\left|J_C(\pi') - J_{C}^{\mu, \pi}(\pi')\right| \leq \frac{4\epsilon_C\gamma}{(1-\gamma)^2}D(\mu, \pi')D(\pi, \pi').
\end{equation}
As the range of $C$ is in $\mathbb{R}_{\geq0}$,
\begin{equation}
\small
\label{eq:cost surrogate inequality}
\begin{aligned}
&J_{C}^{\mu, \pi}(\pi')^2 \leq (J_C(\pi') + \frac{4\epsilon_C\gamma}{(1-\gamma)^2} D(\mu, \pi')D(\pi, \pi'))^2. \\
\Rightarrow& J_{C}^{\mu, \pi}(\pi')^2 - J_C(\pi')^2 \leq (\frac{4\epsilon_C\gamma}{(1-\gamma)^2} D(\mu, \pi')D(\pi, \pi'))^2 \\
&\qquad\qquad\qquad\quad + \frac{8\epsilon_C\gamma}{(1-\gamma)^2}D(\mu, \pi')D(\pi, \pi') J_C(\pi'). \\
\end{aligned}
\end{equation}
Using (\ref{eq:cost surrogate inequality}) and Lemma \ref{lemma:square upper bound}, the following inequality holds:
\begin{equation}
\small
\label{eq:std inequality}
\begin{aligned}
&\sqrt{J_S(\pi') - J_C(\pi')^2} - \sqrt{J_S^{\mu, \pi}(\pi') - J_C^{\mu, \pi}(\pi')^2} \\
& = \frac{J_S(\pi') - J_C(\pi')^2 - (J_S^{\mu, \pi}(\pi') - J_C^{\mu, \pi}(\pi')^2)}{\sqrt{J_S(\pi') - J_C(\pi')^2} + \sqrt{J_S^{\mu, \pi}(\pi') - J_C^{\mu, \pi}(\pi')^2}} \\
&\leq \frac{(J_S(\pi') - J_S^{\mu, \pi}(\pi')) + (J_C^{\mu, \pi}(\pi')^2 - J_C(\pi')^2)}{\mathrm{CVaR}_{\alpha}(C_{\pi'})} \\
& \leq \left(\left(\frac{4\epsilon_C\gamma}{(1-\gamma^2)}\right)^2D(\mu, \pi')D(\pi, \pi') +\frac{2\epsilon_S\gamma^2}{(1-\gamma^2)^2} \right. \\
&\left.+ \frac{8\epsilon_C\gamma}{(1-\gamma)^2}J_{C}(\pi')\right)D(\mu, \pi')D(\pi, \pi')/{\mathrm{CVaR}_{\alpha}(\pi')} \\
&= \epsilon_{\mathrm{CVaR}}D(\mu, \pi')D(\pi, \pi').
\end{aligned}
\end{equation}
Then, using (\ref{eq: cost inequality}) and (\ref{eq:std inequality}),
\begin{equation*}
\small
\begin{aligned}
&\mathrm{CVaR}_{\alpha}(C_{\pi'}) - \overline{\mathrm{CVaR}}_{\alpha}(C_{\pi'}) \\
&= J_C(\pi') - J_C^{\mu, \pi}(\pi') + \frac{\phi(\Phi^{-1}(\alpha))}{\alpha}\left(\sqrt{J_S(\pi') - J_C(\pi')^2} \right. \\
&\quad\left. - \sqrt{J_S^{\mu, \pi}(\pi') - J_C^{\mu, \pi}(\pi')^2}\right) \\
& \leq \left(\frac{4\epsilon_C\gamma}{(1 - \gamma)^2} + \epsilon_{\mathrm{CVaR}}\frac{\phi(\Phi^{-1}(\alpha))}{\alpha}\right)D(\mu, \pi')D(\pi, \pi').
\end{aligned}
\end{equation*}
\end{proof}

\subsection{Trust Region Constraint}
\label{appendix:trust region}
By Pinsker's inequality, the following inequality holds:
\begin{equation*}
\small
\begin{aligned}
D(\pi, \pi') \leq \sqrt{\underset{s}{\mathrm{max}}D_{\mathrm{KL}}(\pi(\cdot|s)||\pi'(\cdot|s))/2}.
\end{aligned}
\end{equation*}
From Appendix B in \cite{meng2021offtrpo}, the following inequality also holds:
\begin{equation*}
\small
\begin{aligned}
D(\mu, \pi') &\leq D(\mu, \pi) + D(\pi, \pi').
\end{aligned}
\end{equation*}
Then, the trust region constraint can be expressed as:
\begin{equation}
\small
\label{eq:trust region inequality}
\begin{aligned}
&D(\mu, \pi')D(\pi, \pi') \leq \underset{s}{\mathrm{max}}D_{\mathrm{KL}}(\pi(\cdot|s)||\pi'(\cdot|s)) \\
&\!\!+ \!\!\sqrt{\underset{s}{\mathrm{max}}D_{\mathrm{KL}}(\mu(\cdot|s)||\pi(\cdot|s))\underset{s}{\mathrm{max}}D_{\mathrm{KL}}(\pi(\cdot|s)||\pi'(\cdot|s))} \leq \delta.
\end{aligned}
\end{equation}
As the maximum operation is difficult to implement in continuous state space settings, we replace it with expectation on trajectories sampled by $\mu$ as follows:
\begin{equation}
\small
\label{eq:trust region ineq}
\begin{aligned}
&D_{\mathrm{KL}}(\pi||\pi') + \sqrt{D_{\mathrm{KL}}(\mu||\pi)D_{\mathrm{KL}}(\pi||\pi')} \leq \delta.
\end{aligned}
\end{equation}
The square root operation on $\pi'$ can lead to infinite value and increase nonlinearity while computing gradients.
Thus, by rearranging (\ref{eq:trust region ineq}), we remove the square root on $\pi'$.
\begin{equation}
\small
\begin{aligned}
&D_{\mathrm{KL}}(\pi||\pi') \leq \delta + D_{\mathrm{KL}}(\mu||\pi)/2 \\
&- \sqrt{D_{\mathrm{KL}}(\mu||\pi)\left(\delta + D_{\mathrm{KL}}(\mu||\pi)/4\right)} = \delta - \delta_{\mathrm{old}}.
\end{aligned}
\end{equation}

\addtolength{\textheight}{-12cm}  
\balance

\bibliographystyle{IEEEtran}
\bibliography{main}

\end{document}